\documentclass[11pt]{article}
\usepackage{ijcai18}
\usepackage{times}
\usepackage{helvet}
\usepackage{courier}
\usepackage{latexsym}
\usepackage{amsfonts,amssymb}
\usepackage{amsmath}
\usepackage{graphicx}
\usepackage{algorithm}
\usepackage{algorithmic}
\usepackage{float}
\usepackage{array}
\usepackage{multirow}
\usepackage{url}
\usepackage{amsmath,amsthm}
\usepackage{extarrows}

\begin{document}

\newcommand{\tabincell}[2]{\begin{tabular}{@{}#1@{}}#2\end{tabular}}

\date{}

\title{Convexification of Neural Graph}
\date{}

\author{Han Xiao$^\spadesuit$ \\
	State Key Lab. of Intelligent Technology and Systems, \\ 
	National Lab. for Information Science and Technology, \\
	Dept. of Computer Science and Technology, Tsinghua  University, Beijing 100084, PR China
	\\ Almighty.Xiao.Han@iCloud.com; 
	\\$^\spadesuit$Correspondence Author: \url{http://www.semantics.top}
}

\maketitle

\begin{abstract}
Traditionally, most complex intelligence architectures are extremely non-convex, which could not be well performed by convex optimization. However, this paper decomposes complex structures into three types of nodes: operators, algorithms and functions. Iteratively, propagating from node to node along edge, we prove that ``regarding the tree-structured neural graph, it is nearly convex in each variable, when the other variables are fixed.'' In fact, the non-convex properties stem from circles and functions, which could be transformed to be convex with our proposed \textit{\textbf{scale mechanism}}. Experimentally, we justify our theoretical analysis by two practical applications.
\end{abstract}

\section{Introduction}
Neural graph\footnote{Notably, if you are not familiar with neural graph, just treat it as usual neural network and ignore any algorithm-related parts of this paper.} is the intelligence architecture, which is characterized by both logics and neurons, \cite{xiao2017ndt}. This definition means that traditional algorithms (e.g. \textit{max flow method, A$^*$ searching, etc.}) could be embedded into neural architectures by the proposed principle of \cite{xiao2017hungarian} as a dynamically graph-constructing process. However, regarding the objective of neural graph, the landscape is still unclear, thus the issue of optimization is temporarily an obstacle towards the complete theoretical foundation. Therefore, in this paper, we analyze the convex property of neural graph and provide a transformation methodology to reform the neural graph as a convex structure. 

Mathematically, there exist three components in neural graph, namely operators, algorithms and functions. Operators indicate four element-wise arithmetic operations (i.e. \textit{plus, minus, multiplication} and \textit{division}), convolution and matrix multiplication. Algorithm means logics-based instructions, while function corresponds to element-wise mathematic mapping (e.g. $x^2, \frac{1}{1+e^{-x}}$). 

On one hand, neural graph is organized as a graph. The node corresponds to operator, algorithm or function, while the edge links from the inputs to the output of corresponding operations. On the other hand, the convexity is characterized by the inequality of second-order form as (\ref{cond}) shows. \textit{\textbf{In essence, if we iteratively ensure the corresponding second-order condition of (\ref{cond}) in each node along the edge as backward propagation, the convex property of entire graph is guaranteed.}}  For the example of Figure \ref{fig:igraph}, to begin, we ensure the condition of (\ref{cond}) for the cross entropy, then we iteratively propagate the guarantee of (\ref{cond}) from objective to variables. In spirit, this process is similar as using backward propagation to work out the second-order form. Specifically, regarding the weight $W$, the satisfaction of (\ref{cond}) means it is convex relative to the objective of cross entropy.

\begin{figure}[H]
	\centering
	\includegraphics[width=\linewidth]{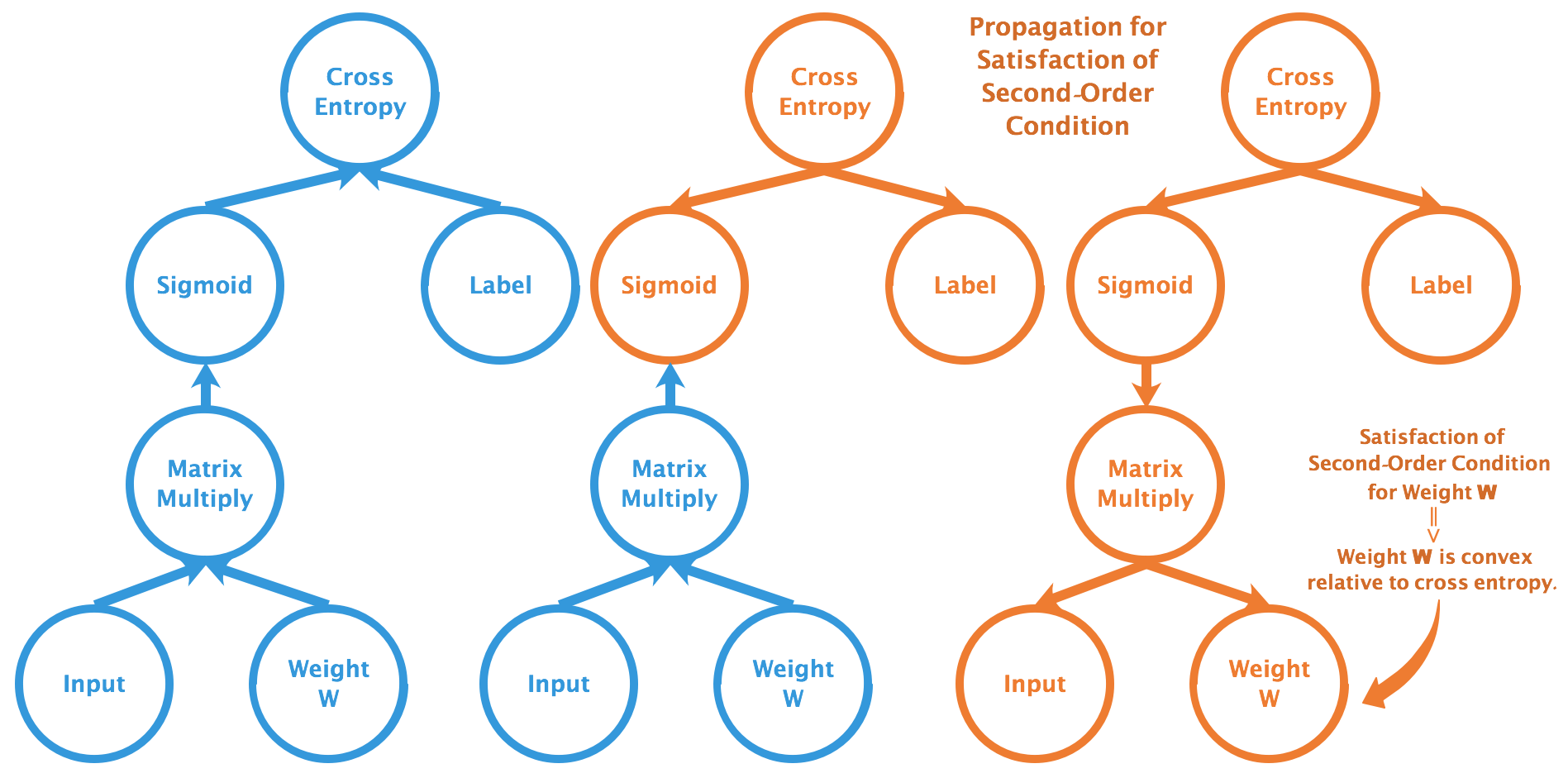}
	\caption{Propagation for Satisfaction of Second-Order Condition (Inequality (\ref{cond})). When the process propagates to the weight, it is proved to be convex relative to the cross entropy.}
	\label{fig:igraph}
\end{figure}

Notably, the algorithm plays a role of selection, which constructs a graph after the forward propagation, \cite{xiao2017hungarian}. Thus, if the constructed graph ensures the convex condition of (\ref{cond}), the objective of neural graph with algorithms is convex to each variable. Besides, neural graph involves few operations of element-wise minus and division, hence we package the corresponding parts into functions. For example, we analyze the square loss $(x-a)^2$ as an entire function, rather than a ``square multiplication'' with a ``minus'' node.

In this paper, we prove that \textbf{\textit{``regarding the tree-structured neural graph, it is convex in each variable, when the other variables are fixed, if the function nodes satisfy the convexification inequality.''}} The circle in neural graph indicates the structure where several variates for a node are functions of a specific symbol. For instance, in the graph of $y=W\sigma(Wx)$, $y$ as a multiplication node has two variates $\sigma(Wx)$ and $W$, both of which are functions of the symbol $W$, making a circle. From our proofs, we could conclude that the non-convex property only originates from circles and non-linear functions. Thus, we propose a factor scaling transformation that \textit{\textbf{scale mechanism}}, which is the technique to modify the neural graph slightly for making it convex.

Finally, as \cite{Kawaguchi2016Deep} suggests, we conjecture that ``every local minimum is nearly a global minimum for conventional neural graph'' without proofs. Empirically, if there are indeed some poor local minimums, the multiple runs of a single graph may lead to many diverse results. However, nearly all the neural architectures could be repeated with reported accuracies, which implies our conjecture. 

\textbf{Contributions.} This paper contributes in two aspects. \textbf{(1.)} We prove the non-convexity of neural graph only stems from circles and non-linear functions. \textbf{(2.)} We provide a simple methodology that scaling the non-linear functions for convexification of neural graph. 

\section{Related Work}
Neural architectures are attractive in terms of its generalization properties, \cite{Mhaskar2016Learning}, while it also introduces non-convex issues in optimization, \cite{Kawaguchi2016Deep}. However, \cite{Katta1987Some} has proved that discovering the global minimum of a general non-convex problem is NP-complete. Thus, we intuitively expect neural graph is nearly convex in nature, which could benefit from convex optimization such as AdaDelta \cite{Zeiler2012ADADELTA}.  

Current researches only focus on linear or non-linear multiple layer perceptions (MLP). For example, \cite{Baldi1988Linear} attempts to study the optimization properties of shadow linear network, while \cite{Goodfellow-et-al-2016} and \cite{Choromanska2015Open} address the conjecture and open problem for deep linear and non-linear network, respectively. Further, \cite{Kawaguchi2016Deep} provides an exact path to these issues with some assumptions.

However, practical neural graphs are more complex than simple MLP, such as LSTM-based models \cite{Wang2017Bilateral}\cite{Yih2013Question} \cite{tan2015lstm} and CNN-based methods \cite{He2015Deep} \cite{Kaiming2016Identity} \cite{Yin2015Convolutional}  \cite{Romero2014FitNets}. Thus, convexity analysis should be extended to conventional situations, which is the target of this paper and is barely focused by our community.
 
\section{Notation, Gradient and Convexity}
\label{ein}
Without loss of generality, we adopt matrix as our basic numeric type for each node. Thus, the second-order item has the form of four-subscripted array, such as $a_{iji'j'}$. For brevity, \textbf{\textit{Einstein's summation convention}} is applied as a protocol for tensor analysis, which specifics an equation without sum symbol (i.e. $\sum$). \textit{\textbf{In detail, we sum out the variables with the subscript that does not exist in the other side.}} For example, 
\begin{eqnarray}
a_{ij}b_{jmn} = c_{imp} \xLongrightarrow{for} \sum_{j,n} a_{ij}b_{jmn} = \sum_{p} c_{imp}
\end{eqnarray}
In this paper, we apply this notation by default.

The objective of neural graph (i.e. $\mathcal{E}$) is a scalar. Hence, the first-order gradient of objective relative to a specific variable node is a matrix, while the second-order is a four-subscripted array, which are defined as below:
\begin{eqnarray}
\frac{\partial \mathcal{E}}{\partial A} &=& \left\{ \frac{\partial \mathcal{E}}{\partial a_{ij}}\right\}_{ij} \in \mathbb{R}^{n \times m} \\ 
\frac{\partial^2 \mathcal{E}}{\partial A^2} &=& \left\{ \frac{\partial^2 \mathcal{E}}{\partial a_{ij} \partial a_{i'j'}} \right\}_{iji'j'} \in \mathbb{R}^{n \times m \times n \times m} \nonumber
\end{eqnarray} 
where $A \in \mathbb{R}^{n \times m}$ is the matrix of node. Based on the definition of gradient, we expand the scalar objective with Taylor series (using Einstein's convention), such as:
\begin{eqnarray}
f(X) & = & f(X_0) + (\frac{\partial \mathcal{E}}{\partial X})_{ij}(X-X_0)_{ij} +  \\
& & (\frac{\partial^2 \mathcal{E}}{\partial X^2})_{iji'j'}(X-X_0)_{ij}(X-X_0)_{i'j'} + \mathcal{O} \nonumber
\end{eqnarray}
Besides, we have the \textit{\textbf{convex definition}} of neural graph of objective function as below \cite{Boyd2013Convex}:
\begin{eqnarray}
f(X) \ge f(X_0) + (\frac{\partial \mathcal{E}}{\partial X})_{ij}(X-X_0)_{ij} 
\end{eqnarray}
which implies that the second-order item (with Einstein convention) is larger than zero as 
\begin{eqnarray}
(\frac{\partial^2 \mathcal{E}}{\partial X^2})_{iji'j'}(X-X_0)_{ij}(X-X_0)_{i'j'} \ge 0
\end{eqnarray}
We reform this critical second-order condition in the form of matrix as:
\begin{eqnarray}
\mathbf{(X-X_0)}_j (\frac{\partial^2 \mathcal{E}}{\partial A^2})_{jj'} \mathbf{(X-X_0)}_{j'}^T \ge 0 
\label{cond}
\end{eqnarray}
where $\mathbf{(X-X_0)}_j$ is a vector of $\mathbb{R}^{1 \times n}$ indexed by $j$, and $(\frac{\partial^2 \mathcal{E}}{\partial A^2})_{jj'}$ is a matrix of $\mathbb{R}^{n \times n}$ indexed by $j,j'$. Notably, there exist a sum symbol with the subscript $j,j'$, that is omitted. However, for clarity, we present the condition (\ref{cond}) with the full formulation as below. We call this inequality as \textit{\textbf{convex condition}}.
\begin{eqnarray}
\sum_{j=1}^{m} \sum_{j'=1}^{m} \mathbf{(X-X_0)}_j (\frac{\partial^2 \mathcal{E}}{\partial A^2})_{jj'} \mathbf{(X-X_0)}_{j'}^T \ge 0 
\end{eqnarray}

\section{Proof of Convexification}
In this section, we sequentially prove that operators, functions and algorithms guarantee the condition of (\ref{cond}), if the successor node has ensured this condition. 

\newtheorem{theorem}{Theorem}

\begin{theorem}
Assuming there is no circle in graph and the successor node ensures the condition of (\ref{cond}), the nodes of element-wise plus operator $+$, element-wise multiplication operator $\otimes$, convolution $*$ and matrix multiplication operator also guarantee this condition.
\end{theorem}

\begin{proof} 
	Notably, we employ Einstein's convention for a unified representation of operators, which refers to Section \ref{ein}.
	
	\textbf{(1.)} Firstly, we prove the case of plus. Let $c_{ij} = a_{ij} + b_{ij}$, we have:
	\begin{eqnarray}
		(\frac{\partial^2 \mathcal{E}}{\partial a^2})_{iji'j'} = (\frac{\partial^2 \mathcal{E}}{\partial c^2})_{iji'j'}
	\end{eqnarray}
	Therefore, the conclusion is obvious.
	
	\textbf{(2.)} Secondly, we discuss the case of element-wise multiplication. Let $c_{ij} = a_{ij} b_{ij}$. Since there is no circle that $b_{ij}$ is not a function of $a_{ij}$, we have:
	\begin{eqnarray}
	(\frac{\partial \mathcal{E}}{\partial a})_{ij} = (\frac{\partial \mathcal{E}}{\partial c})_{ij}  b_{ij},
	\;\; (\frac{\partial^2 \mathcal{E}}{\partial a^2})_{iji'j'} = (\frac{\partial^2 \mathcal{E}}{\partial c^2})_{iji'j'} b_{ij}  b_{i'j'}
	\end{eqnarray}
	Substitute into the right side of (\ref{cond}).
	\begin{eqnarray}
	& &	\mathbf{(X-X_0)}_j (\frac{\partial^2 \mathcal{E}}{\partial a^2})_{jj'} \mathbf{(X-X_0)}^T_{j'} \label{em-last} \\
	& = &	(\mathbf{b}_j \otimes \mathbf{(X-X_0)}_j) (\frac{\partial^2 \mathcal{E}}{\partial c^2})_{jj'}
	(\mathbf{(X-X_0)}_{j'}^T \otimes \mathbf{b}_{j'}^T) \ge 0 \nonumber 
	\end{eqnarray}
 	The inequation of (\ref{em-last}) holds because of the arbitrariness of vector $\mathbf{(X-X_0)}$.
 	
 	\textbf{(3.)} Thirdly, we explore the convolution. Let $c_{ij} = a_{pq}*b_{mn}$. Since there is no circle that $b_{mn}$ is not a function of $a_{pq}$, we have:
 	\begin{eqnarray}
 	& (\frac{\partial \mathcal{E}}{\partial a})_{pq} = (\frac{\partial \mathcal{E}}{\partial c})_{ij} * b_{mn} \\
 	& (\frac{\partial^2 \mathcal{E}}{\partial a^2})_{pqp'q'} = (\frac{\partial^2 \mathcal{E}}{\partial \nonumber c^2})_{iji'j'} * b_{mn} * b_{m'n'} 
 	\end{eqnarray}
 	where the convolutions of $b_{mn}, b_{m'n'}$ respectively make effect on the corresponding dimensions of the second-order of $c$.
 	Substitute into the right side of (\ref{cond}):
 	\begin{eqnarray}
 	& &	\mathbf{(X-X_0)}_q (\frac{\partial^2 \mathcal{E}}{\partial a^2})_{qq} \mathbf{(X-X_0)}^T_{q'} \label{conv-last} \\
 	& = &	(\mathbf{(X-X_0)*B})_j (\frac{\partial^2 \mathcal{E}}{\partial c^2})_{jj'} (\mathbf{(X-X_0)*B})_{j'}^T \ge 0 \nonumber
 	\end{eqnarray}
 	Because of the arbitrariness of vector $\mathbf{(X-X_0)}$, our conclusion of (\ref{conv-last}) is established.
 	
 	\textbf{(4.)} Finally, we deal with the situation of matrix multiplication. Let $c_{ij} = a_{ik} b_{kj}$. Since there is no circle that $b_{ij}$ is not a function of $a_{ij}$, we have:
	\begin{eqnarray}
	(\frac{\partial \mathcal{E}}{\partial a})_{ik} = (\frac{\partial \mathcal{E}}{\partial c})_{ij} b_{kj}, \;\;
	(\frac{\partial^2 \mathcal{E}}{\partial a^2})_{iki'k'} = (\frac{\partial^2 \mathcal{E}}{\partial \nonumber c^2})_{iji'j'}  b_{kj}  b_{k'j'} 
	\end{eqnarray}
	Substitute into the right side of (\ref{cond}):
	\begin{eqnarray}
	& &	\mathbf{(X-X_0)}_k (\frac{\partial^2 \mathcal{E}}{\partial a^2})_{kk'} \mathbf{(X-X_0)}^T_{k'} \label{mm-last} \\
	& = &	(\mathbf{(X-X_0)B})_j (\frac{\partial^2 \mathcal{E}}{\partial c^2})_{jj'} (\mathbf{(X-X_0)B})_{j'}^T \ge 0 \nonumber
	\end{eqnarray}
	In the inequation of (\ref{mm-last}), $\mathbf{B}$ is the matrix of $\{b_{kj}\}$ and $\mathbf{(X-X_0)} = \{(X-X_0)_{ik}\}$. Because of the arbitrariness of vector $\mathbf{(X-X_0)}$, the proof is end.
\end{proof}
Intuitively, this theorem represents that operator nodes keep the convexity in an iterative manner along the graph edge. Next, we will prove the situation of functions. Without loss of generality, we treat the functions in batch mode, where each column corresponds to an independent sample. This assumption accords to the practical situations, which does not harm the conclusion of this paper.

\begin{theorem}
	Assuming the successor node ensures the condition of (\ref{cond}) and each column corresponds to an independent sample, the function node $c=s(X)$ also guarantees this condition if the \textbf{convexification inequality} holds as 
	\begin{eqnarray}
	\forall j \in [1...m]:\;\;\lambda_{min}(\frac{\partial^2 \mathcal{E}}{\partial c^2})_{jj} + \min_{i=1}^n \frac{(\frac{\partial \mathcal{E}}{\partial c})_{i; j}s''_{ii; jj}}{(s'_{i;j})^2} \ge 0
	\end{eqnarray}
\end{theorem}

\begin{proof}
	Let $c_{ij} = s(a_{ij})$, where $s$ is the element-wise function, such as $sigmoid, tanh$. The gradients are as below:
	\begin{small}
	\begin{eqnarray}
	& (\frac{\partial \mathcal{E}}{\partial a})_{ij} = (\frac{\partial \mathcal{E}}{\partial c})_{ij}s'(a_{ij}) \label{t2e1} \\ 
	& (\frac{\partial^2 \mathcal{E}}{\partial a^2})_{iji'j'} = (\frac{\partial^2 \mathcal{E}}{\partial c^2})_{iji'j'}s'(a_{ij})s'(a_{i'j'}) + (\frac{\partial \mathcal{E}}{\partial c})_{ij}s''(a_{ij})_{i'j'} \nonumber
	\end{eqnarray}
	\end{small}
	We notate $s'(a_{ij}) \doteq s'_{ij}$ and $s''(a_{ij})_{i'j'} \doteq s''_{iji'j'}$. Because $a_{ij}$ is unrelated to $a_{i'j'}$ when $i \neq i', j \neq j'$, the $s''_{iji'j'}$ is only non-zero regarding the position of $i=i', j=j'$. Also, assuming each column corresponds to an independent sample, $ (\frac{\partial^2 \mathcal{E}}{\partial c^2})_{iji'j'}$ and $ (\frac{\partial^2 \mathcal{E}}{\partial a^2})_{iji'j'}$ is only non-zero when $j=j'$. Thus, we reformulate the (\ref{t2e1}) given the fixed $j$ as:
	\begin{eqnarray}
	& (\frac{\partial^2 \mathcal{E}}{\partial a^2})_{ii; jj} = (\frac{\partial^2 \mathcal{E}}{\partial c^2})_{ii;jj}s'_{i;j}s'_{i;j} + (\frac{\partial \mathcal{E}}{\partial c})_{i; j}s''_{ii; jj} \label{19}
	\end{eqnarray}
	Note that, $(\frac{\partial^2 \mathcal{E}}{\partial c^2})_{jj}$ satisfies (\ref{cond}), which implies it is a positive semi-defined matrix, because $j$ indicates one independent sample. Thus, we perform the eigenvalue matrix decomposition as $(\frac{\partial^2 \mathcal{E}}{\partial c^2})_{jj} = \mathbf{U_j^T\Lambda_j U_j}$. Besides, the second item in (\ref{19}) could be reformed as a diagonal matrix $diag((\frac{\partial \mathcal{E}}{\partial c})_{i; j}s''_{ii; jj})$. Hence, we have:
	\begin{eqnarray}
	& (\frac{\partial^2 \mathcal{E}}{\partial a^2})_{ii; jj} = \mathbf{U_j^T} (\Lambda_{ii;j} s'_{i;j}s'_{i;j} + diag((\frac{\partial \mathcal{E}}{\partial c})_{i; j}s''_{ii; jj}))\mathbf{U_j} \label{20} \nonumber \\
	\end{eqnarray}
	With the satisfaction of convexification inequality, we have proved this theorem.
\end{proof}

\begin{theorem}
	Algorithms do not modify the convexity of neural graph, or the satisfaction of inequality (\ref{cond}).
\end{theorem}
\begin{proof}
	Algorithms play a role of selection as dynamically constructing neural graph. If the constructed graph satisfies the convex condition of (\ref{cond}), the neural graph with algorithms would ensure the conclusion. Specifically, the constructed graph is only composed by operators and functions, which are characterized by Theorem 1 and 2. Thus, algorithms do not modify the convexity of neural graph.
\end{proof}

\begin{theorem}
	If the objective nodes of neural graph satisfy the condition of (\ref{cond}) and the convexification inequality holds for every function node, the objective of entire tree-structured neural graph is convex.
\end{theorem}

\begin{proof}
Previously stated, neural graph is composed by three parts: operators, algorithms and functions. Algorithm dynamically constructs the graph, thus does not modify the convexity. Operators (i.e. \textit{element-wise plus/multiplication, convolution and matrix multiplication}) guarantee the convex conditions of (\ref{cond}) if the successor node is ensured and there is no circle. Specifically in this theorem, if the objective nodes ensure the condition of (\ref{cond}), the conclusion could be deducted iteratively from the objective to the input/variable nodes along the path, which establishes the proof. 
\end{proof}

\newtheorem{lem}{Lemma}

\begin{lem}
	The square loss satisfies the condition of (\ref{cond}).
\end{lem}
\begin{proof}
	The square loss has a form of $l=\frac{1}{2}\sum_{ij}(y_{ij}-\hat{y}_{ij})^2$, where $y$ is the output of neural graph and $\hat{y}$ is the corresponding label. Notably, each $j$ corresponds to an independent sample, which means the second-order gradient of matrix form as $l''_{ii;jj'}$ is only non-zero when $j=j'$. Thus, the first-order gradient is $l'_{i,j} = (y_{ij} - \hat{y}_{ij})$, then the second-order gradient is $l''_{ii';jj'}=1$ only when $i=i'$, otherwise it is $0$.  In conclusion, the condition of (\ref{cond}) is ensured.
\end{proof}

\begin{lem}
	The absolute loss satisfies the condition of (\ref{cond}).
\end{lem}
\begin{proof}
	The absolute loss has a form of $l=\sum_{ij}|y_{ij}-\hat{y}_{ij}|$, where $y, \hat{y}$ indicate the output and ground-truth label of neural graph, respectively. Similarly, the first-order gradient is $l'_{i;j}=sign(y_{ij} - \hat{y}_{ij})$ where $sign$ is the sign function, then the second-order gradient is $l''_{ii';jj'} = 0$. The situations is similar to Lemma 1, which ends the proof.
\end{proof}

\begin{lem}
	The cross entropy satisfies the condition of (\ref{cond}).
\end{lem}
\begin{proof}
	Regarding the second-order form, $l'_{ii';jj'}$ is only non-zero when $j=j'$, because $j$ corresponds to one independent sample. Thus, we omit the subscript $j$. The cross entropy has a form as $-\sum_i \hat{y}_i log (y_i) + (1-\hat{y}_i) log (1 - y_i)$, where $y, \hat{y}$ indicate the output and ground-truth label of neural graph, respectively. Thus, the first-order gradient is $l'_i = -\frac{\hat{y}_i}{y_i} + \frac{1-\hat{y}_i}{1-y_i}$, while the second-order gradient is $l''_{ii'} = \frac{\hat{y}_i}{y_i^2} + \frac{1-\hat{y}_i}{(1-y_i)^2}$ when $i=i'$, otherwise $0$. Because $y_i, \hat{y}_i \in [0,1]$, the second-order gradient is non-zero.  In conclusion, cross entropy satisfies the condition of (\ref{cond}).
\end{proof}

\begin{theorem}
	If there exists no circle in graph and the convexification inequality holds for every function node, the objective of neural graph with conventional loss (i.e. \textit{square, absolute and cross entropy}) is convex to each variable.
\end{theorem}

\section{Scale Mechanism for Convexification}
Scale mechanism is the technique to modify the neural graph slightly for making it convex. The non-convex property stems from circles and functions, which could be modified for convexification by factor scaling in the manner of this section.

\subsection{Circles}
Circles could not exist in our framework, because there are more residual items in the second-order form. To begin, we notice there exist many paths from the loss to a symbol $a$, the successor nodes of which are noted as $p_1, p_2, ...p_k$. Thus, the second-order of $a$ could be composed by many similar items, such as (using Einstein's convention):
\begin{eqnarray}
\frac{\partial^2 \mathcal{E}}{\partial a^2} & = & \frac{\partial}{\partial a}(\frac{\partial \mathcal{E}}{\partial p_r} \frac{\partial p_r}{\partial a}) \\
& = & \frac{\partial}{\partial p_s}(\frac{\partial \mathcal{E}}{\partial p_r} \frac{\partial p_r}{\partial a}) \frac{\partial p_s}{\partial a} \\
& =& \frac{\partial^2 \mathcal{E}}{\partial p_s \partial p_r} \frac{\partial p_r}{\partial a} \frac{\partial p_s}{\partial a} + \frac{\partial \mathcal{E}}{\partial p_r} \frac{\partial p_s}{\partial a} \frac{\partial^2 p_r}{\partial p_s \partial a}
\end{eqnarray}
There could be many possible forms of $p_s, p_r$, but conventional neural graph only takes the recursive form such as RNN in the case of circle. Simply, $p_s \rightarrow p_r \rightarrow ...$ if $s < r$:
\begin{eqnarray}
\frac{\partial^2 \mathcal{E}}{\partial a^2} & = & \sum_{r=s}  \frac{\partial^2 \mathcal{E}}{\partial p_r^2} \frac{\partial p_r}{\partial a} \frac{\partial p_r}{\partial a} \\
& + &  \mathop{\underline{\sum_{s < r} \left( 2(\frac{\partial^2 \mathcal{E}}{\partial p_r^2} \frac{\partial p_r}{\partial a} \frac{\partial p_s}{\partial a} \frac{\partial p_r}{\partial p_s}) + 
\frac{\partial \mathcal{E}}{\partial p_r} \frac{\partial p_s}{\partial a} \frac{\partial^2 p_r}{\partial p_s \partial a} \right)}}_{Residual}  \nonumber
\end{eqnarray}
For the first term, we have discussed the situation in Theorem 1 and 2. Regarding the second term, we transform $p_r(p_s)$ into $p_r(\delta p_s)$, thus the residual term is as below:
\begin{eqnarray}
\delta\sum_{r > s} \left( 2(\frac{\partial^2 \mathcal{E}}{\partial p_r^2} \frac{\partial p_r}{\partial a} \frac{\partial p_s}{\partial a} \frac{\partial p_r}{\partial p_s}) + 
		\frac{\partial \mathcal{E}}{\partial p_r} \frac{\partial p_s}{\partial a} \frac{\partial^2 p_r}{\partial p_s \partial a} \right)å
\end{eqnarray}
If the scaling factor is sufficiently small, the residual item is sufficiently insignificant or the main component is sufficiently advantageous. In this manner, the circle in the recursive case guarantees the convex condition of (\ref{cond}). \textit{\textbf{Notably, we could multiply the scaling factor in any node along the path from $b$ to $a$.}}

\subsection{Functions}
There are three conventional non-linear functions in neural graph: \textit{ReLU, sigmoid} and \textit{tanh}, which will be discussed in this subsection with convexification inequality. Notably, novel non-linear functions could be analyzed in the same manner.

\textbf{(1.) ReLU.} ReLU (Rectified Linear Units, \cite{Glorot2012Deep}) takes the form of $relu(x) = max(x, 0)$, the second-order item of which is zero. Thus, ReLU naturally satisfies the convexification inequality and ensures the condition of (\ref{cond}).

\textbf{(2.) Sigmoid.} Sigmoid is the most classical non-linear function as $\sigma(x) = \frac{1}{1+e^{-x}}$. The first- and second-order gradient items are $\sigma' = \sigma(1-\sigma)$ and $\sigma''=(1-2\sigma)\sigma'$, respectively. Thus, the convexification inequality is as below:
\begin{eqnarray}
\forall j \in [1...m]:\;\;\lambda_{min}(\frac{\partial^2 \mathcal{E}}{\partial c^2})_{jj} + \min_{i=1}^n \frac{(\frac{\partial \mathcal{E}}{\partial c})_{i; j}(1-2\sigma_i)}{\sigma_i(1-\sigma_i)} \ge 0
\end{eqnarray}
Actually, the domain of $\sigma$ is often $[-a, +a]$ (e.g. $[-1,+1]$), thus the value range is $[\sigma(-a),\sigma(a)]$. What we expect is $\sigma$ approaches $\frac{1}{2}$, where $x$ of $\sigma(x)$ approaches $0$. Thus, we scale $\sigma(x)$ as $\sigma(\delta x)$, the domain of which is $[-\delta a, \delta a]$. If the scaling factor $\delta$ is sufficiently small, the domain is nearly a neighbor of origin, where the second item of convexification inequality is nearly insignificant. In this manner, the condition of (\ref{cond}) is guaranteed.

\textbf{(3.) Tanh.} Tanh is the other classical non-linear function as $\sigma(x)=\frac{e^x-e^{-x}}{e^x+e^{-x}}$, the analysis of which is similar with \textit{sigmoid}. Firstly, we calculate the first- and second-order gradient items as $\sigma'=1-\sigma^2$ and $\sigma'' = -2 \sigma\sigma'$. Secondly, we present the convexification inequality:
\begin{eqnarray}
\forall j \in [1...m]:\;\;\lambda_{min}(\frac{\partial^2 \mathcal{E}}{\partial c^2})_{jj} + \min_{i=1}^n \frac{(\frac{\partial \mathcal{E}}{\partial c})_{i; j}(-2\sigma_i)}{(1-\sigma_i^2)} \ge 0
\end{eqnarray}
Finally, we scale the non-linear function as $\sigma(\delta x)$, the domain of which is $[-\delta a, + \delta a]$ (where the original domain is $[-a,+a]$). Because the domain scales into the neighbor of origin, $\sigma$ approaches $0$, which makes the second item insignificant. In this manner, the condition of (\ref{cond}) is ensured.

\subsection{Why Scale Works?}
To demonstrate the convexification inequality and factor scaling mechanism,  we take an example of $f(x)=sin^2(x)$. Regarding the node \textit{``sin''}, $\frac{\partial^2 \mathcal{E}}{\partial (sin(x))^2} = 2$, $\frac{\partial \mathcal{E}}{\partial sin(x)} = 2sin(x)$, $sin''(x) = -sin(x)$ and $sin'(x)=cos(x)$. Thus, we get the convexification inequality such as: 
$$ 2 + \frac{-2sin^2(x)}{cos^2(x)}  \ngeq 0, \;\; x \in [-2, +2]$$
where $x \rightarrow \pm 2$ challenges the inequality, which means the composed function is possibly non-convex. However, when we calculate the scaling factor as $\delta = 0.3$, the scaling leads to a convex function in $[-2, +2]$ as shown in Figure \ref{fig:cie}. \textit{\textbf{Essentially, factor scaling weakens the effect of inner non-convex function, thus outer convex parts could dominate the convexity.}} Specifically, the non-convexity of inner ``$sin(x)$'' is depressed relatively to the outer convex ``$x^2$''.

\begin{figure}[H]
	\centering
	\includegraphics[width=0.85\linewidth]{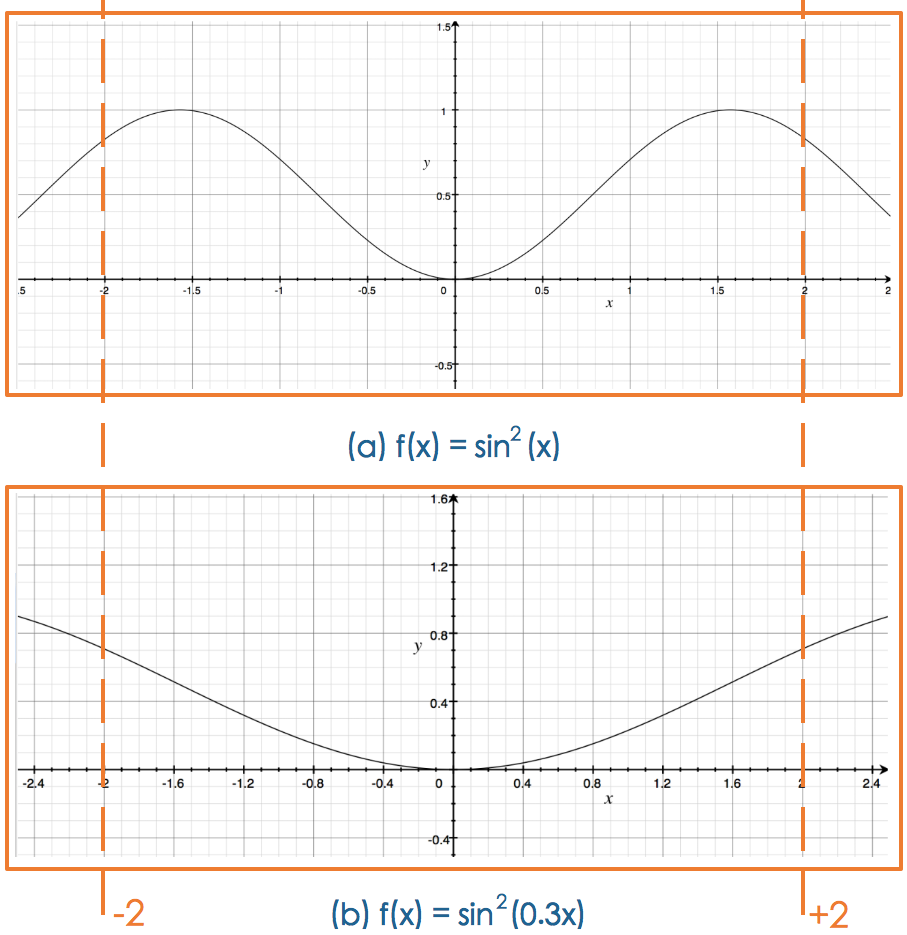}
	\caption{Illustration of Convexification Inequality and Factor Scaling Mechanism.}
	\label{fig:cie}
\end{figure}

It is worthy to note that convexification inequality is a sufficient condition for convexity, not a necessary one. Thus, the practical scaling factor only needs to be sufficiently small but does not need to strictly satisfy this inequality.

\section{Experiment}
To verify the scale mechanism, we conduct a practical MLP-based graph for image recognition on the dataset of MNIST \cite{Lecun1998Gradient} and LSTM-based graph for the task of sentence matching on the dataset of ``Quora Paraphrase Dataset'' \cite{xiao2017hungarian}. Besides, we also have conducted two tasks of variance reduction and faster convergence to testify our theoretical analysis. 

\subsection{Experimental Setting}
\textbf{Dataset.} To testify our theory in a practically large-scale scenario, we choose two datasets: MNIST and ``Quora Paraphrase Dataset''. The MNIST dataset \cite{Lecun1998Gradient} is a classic benchmark dataset, which consists of handwritten digit images, 28 x 28 pixels in size, organized into 10 classes (0 to 9) with 60,000 training and 10,000 test samples. ``Quora Question Pairs''\footnote{The url of the dataset: \url{https://data.quora.com}} \cite{xiao2017hungarian} aims at paraphrase identification. Specifically, There are over 400,000 question pairs in this dataset, and each question pair is annotated with a binary value indicating whether the two questions are paraphrase of each other or not. We choose 10,000 samples respectively as developing/testing set and select 40,000 instances as the training set.

\textbf{Baseline \& Scale Mechanism.} Different graphs are processed differently. Thus, we discuss the baselines and the corresponding scale mechanism for the specific method. Notably, both the baselines are applied the AdaDelta \cite{Zeiler2012ADADELTA} with moment as $0.6$ and regularization as $1.0 \times 10^{-6}$. 

Regarding MNIST, we apply a Multiple Layer Perception (MLP) with structures \textit{``input-300-100-output''} and \textit{sigmoid} as non-linear function. There is no circle in this case, thus we directly scale all the non-linear functions with different factors $\delta$. We train the model until convergence but at most 1,000 rounds and attempt five times of training-testing run to obtain the accuracy and variance.

Regarding ``Quora'', we apply the Siamese LSTM \cite{Wang2017Bilateral} as our baseline. Specifically, Siamese LSTM encodes the input two sentences into two sentence vectors by LSTM, respectively, \cite{Wang2016Semi}. Based on the two sentence vectors, a cosine similarity is leveraged to make the final decision. Specifically, we initialize the word embedding with 300-dimensional GloVe \cite{Pennington2014Glove} word vectors pre-trained in the 840B Common Crawl corpus \cite{Pennington2014Glove} and then set the hidden dimension as 100 for each LSTM. 

There are many recursive circles in LSTM, where the $\mathbf{Wh}$ in each gate is a classic circle, because $\mathbf{h}$ is a function of $\mathbf{W}$. Fortunately, all the paths from $\mathbf{W}$ to $\mathbf{h}$ pass the non-linear functions, which means, only scaling the non-linear function could fix the issues of both circles and functions.

\subsection{Variance Reduction} 
Actually, multiple runs with random initialization could discover various saddle/local/global minimums, which generates the variance of accuracy. Scale mechanism could invoke a more convex surface, which could be expected to reduce the variance of accuracy in theory. In fact, stabilization is critical to industrial application, which is the significance of accuracy variance reduction.

Regarding the experimental protocol, we have tried four hyper-parameters of scaling factor $\delta$: baseline $1.0$, slight scale $0.9$, medium scale $0.5$ and destroying scale $0.3$. The corresponding results are present in Table \ref{performance}. Thus, we have concluded as:
\begin{enumerate}
	\item The standard deviation of medium setting is better than the baseline, which justifies our theory. Specifically, the corresponding objective is more convex. In this way, the scale mechanism could stabilize the solution to some extent.
	\item Destroying setting is an extreme case, where the loss surface could be very singular, which increases the variance.
	\item It is also worthy to note that the trade-off of scaling factor and non-linear function should be carefully tuned. Specifically, smaller scale could lead to more convex surface but make an uncertainty effect on accuracy. We suggest to apply a medium scaling factor, which takes care of both convexity and accuracy.
\end{enumerate}

\begin{table}[H]
	\caption{Performance for MNIST with different factors $\delta$.}
	\centering
	\renewcommand\arraystretch{1.1}
	\label{performance}
	\begin{tabular}{c|c|c|c|c}
		\hline $\delta$ & \textbf{Accuracy} & \textbf{Std.} ($10^{-3}$) & \textbf{Max.} & \textbf{Min.} \\
		\hline 
		\hline 1.0 & \textbf{0.9720} & 4.649 & \textbf{0.9770} & 0.9641 \\
		\hline 0.9 & 0.9717 & 3.835 & 0.9752 & 0.9693 \\
		\hline 0.5 & 0.9709 & \textbf{1.006} & 0.9719 & \textbf{0.9695} \\
		\hline 0.3 & 0.9624 & 2.649 & 0.9695 & 0.9607 \\
		\hline
	\end{tabular}
\end{table}

\subsection{Faster Convergence}
In theory, non-convexity would make the path from initialization to saddle/local/global minimums very bumpy, which slows the convergence down to some degree. Scale mechanism could generate a more convex objective, which could smooth the path of gradient descent methods. In this manner, each step of gradient descent method makes more effects, thus a faster convergence rate could be expected.

\begin{table}[H]
	\caption{Performance for Quora with different factors $\delta$, where ``Conv.'' is short for convergence. }
	\centering
	\renewcommand\arraystretch{1.1}
	\label{converagence}
	\begin{tabular}{c|c|c|c|c}
		\hline $\delta$ & 1.0 & 0.8 & 0.5 & 0.1 \\
		\hline 
		\hline \textbf{Accuracy} & 0.7273 & 0.7274 & 0.7290 & \textbf{0.7301} \\
		\hline \textbf{Conv. Epoch} & 12 & \textbf{10} & 11 & 11 \\
		\hline
	\end{tabular}
\end{table}

Regarding the experimental protocol, we have tried four hyper-parameters of scaling factor $\delta$: baseline $1.0$, slight scale $0.8$, medium scale $0.5$ and extreme scale $0.1$, while we test the model in each training round and record the accuracies. The corresponding results are present in Figure \ref{fig:faster} and Table \ref{converagence}. Thus, we have concluded as:

\begin{enumerate}
	\item The blue line is the baseline, while other lines perform much better than this line in the pre-converged period, which means the scale mechanism indeed speeds the convergence up. This result justifies our theoretical analysis.
	\item The converged epoch of baseline is largest, which means scale mechanism really accelerates the training process.
	\item The purple line (i.e. $\delta=0.1$) always performs better than the baseline and achieves a significantly better accuracy than other settings. This result demonstrates the convexity constructed by scale mechanism could lead to a better local or even the global minimum. 
\end{enumerate}

\begin{figure}[H]
	\centering
	\includegraphics[width=0.95\linewidth]{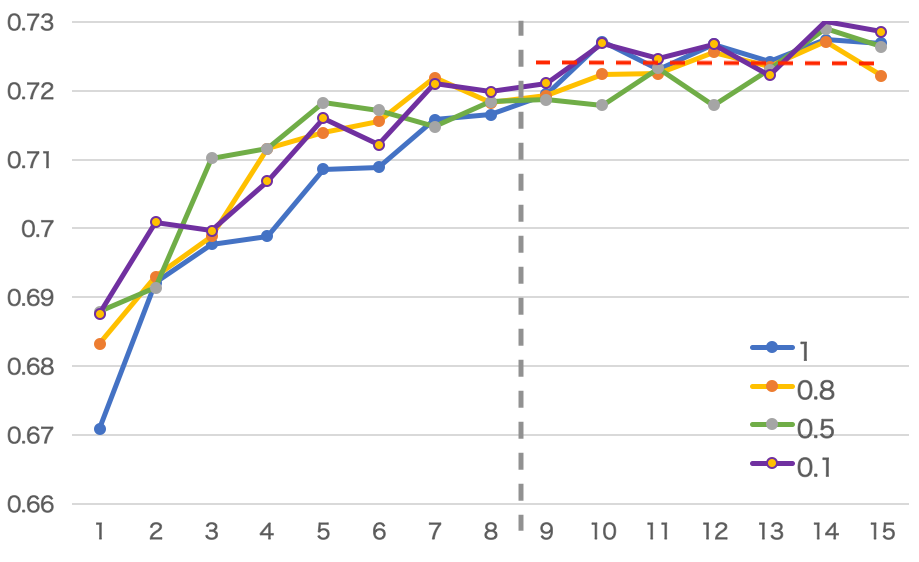}
	\caption{The accuracies of each rounds. The left of gray dashed line is the pre-converged period, while the right part corresponds to the converged period. }
	\label{fig:faster}
\end{figure}


\section{Conjectured Landscape}
In the end of this paper, we formally propose \textbf{\textit{the conjecture for the landscape of neural graph}}, which says that the objective of neural graph is nearly convex from an overview perspective and each local minimum is nearly a global minimum.

Regarding the first statement, we prove the conventional neural graph with scale mechanism is indeed nearly convex in the theoretical and experimental manner. In fact, from our proof, we could intuitively conclude that the neural graph without scale mechanism could still be approximately convex. We hope this conclusion could be generalized to the general case, which is beyond tree-structured and recursive graph. 

Regarding the second statement, we formulate this concept with concentration inequality, such as 
\begin{eqnarray}
	\mathcal{P}(max_{x_i \in C} |f(x_i) - f(x_{*})| > \epsilon) < Ae^{-B\epsilon}
\end{eqnarray} 
where $C$ is the set of local minimums, $x_{*}$ is the global minimum, $f$ is the objective of neural graph and $A, B$ are some parameters independent with $f, x, \epsilon$. We conjecture this inequality is satisfied for the set of neural graphs without proof. We hope future work could prove/deny this conjecture.

\section{Conclusion}
In this paper, we propose an iterative method to analyze the convexity of neural graph. By the proof, we conclude that only the circles and non-linear functions could invoke non-convexity. Thus, we design the scale mechanism to transform a neural graph into a convex form. Experimentally, we demonstrate the scale mechanism could stabilize the accuracy, promote the probability for arriving at the global minimum and speed up the convergence, which justifies our theory.

\newpage

\bibliography{Ref}
\bibliographystyle{named}

\end{document}